\def\year{2018}\relax
\documentclass[letterpaper]{article} %DO NOT CHANGE THIS

\usepackage{amsmath}
\usepackage{amssymb}
\usepackage{amsthm}%Cui Zhen
\usepackage{cases}%Cui Zhen
\usepackage{bm}%Cui Zhen
\usepackage{algorithm}%Cui Zhen
\usepackage{algorithmic}%Cui Zhen
\usepackage{cases}%Cui Zhen
\usepackage{subfigure}%%cuizhen
\usepackage{multirow}%%cuizhen
\usepackage{xcolor} %% cuizhen
\usepackage{enumerate}%% cuizhen
\usepackage{etoolbox}%% cuizhen
\usepackage{boxedminipage} %% cuizhen
\usepackage{booktabs} %% cuizhen
\usepackage{graphicx}
\newcommand{\MyMapTemplatePrefix}[4]{\expandafter#1\csname#3#4\endcsname{#2{#4}}}
\newcommand{\MyMapTemplatePrefixNew}[5]{\expandafter#1\csname#4#5\endcsname{#2{#3{#5}}}}
\forcsvlist{\MyMapTemplatePrefix {\def} {\mathbf} {}} {A,B,C,D,E,F,G,H,I,J,K,L,M,N,O,P,Q,R,S,T,U,V,W,X,Y,Z} % \boldsymbol
\forcsvlist{\MyMapTemplatePrefix {\def} {\mathbf} {}} {a,b,c,d,e,f,g,h,i,j,k,l,m,n,o,p,q,r,s,t,u,v,w,x,y,z,1,0}
\forcsvlist{\MyMapTemplatePrefix {\def} {\widetilde} {wt}} {A,B,C,D,E,F,G,H,I,J,K,L,M,N,O,P,Q,R,S,T,U,V,W,X,Y,Z}
\forcsvlist{\MyMapTemplatePrefix {\def} {\widetilde} {wt}} {a,b,c,d,e,f,g,h,i,j,k,l,m,n,o,p,q,r,s,t,u,v,w,x,y,z} % r
\forcsvlist{\MyMapTemplatePrefixNew {\def} {\widetilde}{\mathbf} {tb}} {A,B,C,D,E,F,G,H,I,J,K,L,M,N,O,P,Q,R,S,T,U,V,W,X,Y,Z}
\forcsvlist{\MyMapTemplatePrefixNew {\def} {\widetilde}{\mathbf} {tb}} {a,b,c,d,e,f,g,h,i,j,k,l,m,n,o,p,q,r,s,t,u,v,w,x,y,z}
\forcsvlist{\MyMapTemplatePrefix {\def} {\widehat} {wh}} {A,B,C,D,E,F,G,H,I,J,K,L,M,N,O,P,Q,R,S,T,U,V,W,X,Y,Z}
\forcsvlist{\MyMapTemplatePrefix {\def} {\widehat} {wh}} {a,b,c,d,e,f,g,h,i,j,k,l,m,n,o,p,q,r,s,t,u,v,w,x,y,z}
\forcsvlist{\MyMapTemplatePrefixNew {\def} {\widehat}{\mathbf} {hb}} {A,B,C,D,E,F,G,H,I,J,K,L,M,N,O,P,Q,R,S,T,U,V,W,X,Y,Z}
\forcsvlist{\MyMapTemplatePrefixNew {\def} {\widehat}{\mathbf} {hb}} {a,b,c,d,e,f,g,h,i,j,k,l,m,n,o,p,q,r,s,t,u,v,w,x,y,z}
\forcsvlist{\MyMapTemplatePrefix {\def} {\mathcal}{mc}} {A,B,C,D,E,F,G,H,I,J,K,L,M,N,O,P,Q,R,S,T,U,V,W,X,Y,Z,a,b,c,d,e,f,g,h,i,j,k,l,m,n,o,p,q,r,s,t,u,v,w,x,y,z}
\forcsvlist{\MyMapTemplatePrefix {\def} {\mathbb} {mb}} {A,B,C,D,E,F,G,H,I,J,K,L,M,N,O,P,Q,R,S,T,U,V,W,X,Y,Z}
\forcsvlist{\MyMapTemplatePrefix {\DeclareMathOperator} {} {} } {tr,diag,sgn}
\def\tp{^\top}

\def\ie{\emph{i.e.}}
\def\etal{\emph{et al.}}

\def\eg{\emph{e.g.}}

\def\bphi{\boldsymbol{\phi}}
\def\bPhi{\boldsymbol{\Phi}}
\def\tlam{\tilde{\lambda}}
\def\blam{\boldsymbol{\lambda}}
\def\bLam{\boldsymbol{\Lambda}}
\def\balpha{\boldsymbol{\alpha}}
\newtheorem{thm}{Theorem}%[section]

\newtheorem{prop}{Proposition}
\newcommand{\tabincell}[2]{\begin{tabular}{@{}#1@{}}#2\end{tabular}}%%

\usepackage{aaai18}  %Required
\usepackage{times}  %Required
\usepackage{helvet}  %Required
\usepackage{courier}  %Required
\usepackage{url}  %Required
\usepackage{graphicx}  %Required
\begin{document}
% The file aaai.sty is the style file for AAAI Press
% proceedings, working notes, and technical reports.
%
\title{Spatio-Temporal Graph Convolution for\\ Skeleton Based Action Recognition}
\author{Chaolong Li,$^*$ Zhen Cui,\thanks{Chaolong Li and Zhen Cui have equal contributions.} Wenming Zheng, Chunyan Xu, Jian Yang \\
School of Computer Science and Engineering, Nanjing University of Science and Technology, Nanjing, China\\
School of Biological Science \& Medical Engineering, Southeast University, Nanjing, China\\
\{zhen.cui, cyx, csjyang\}@njust.edu.cn; \{lichaolong, wenming\_zheng\}@seu.edu.cn}

\maketitle
\begin{abstract}
Variations of human body skeletons may be considered as dynamic graphs, which are generic data representation for numerous real-world applications. In this paper, we propose a spatio-temporal graph convolution (STGC) approach for assembling the successes of local convolutional filtering and sequence learning ability of autoregressive moving average. To encode dynamic graphs, the constructed multi-scale local graph convolution filters, consisting of matrices of local receptive fields and signal mappings, are recursively performed on structured graph data of temporal and spatial domain. The proposed model is generic and principled as it can be generalized into other dynamic models. We theoretically prove the stability of STGC and provide an upper-bound of the signal transformation to be learnt. Further, the proposed recursive model can be stacked into a multi-layer architecture. To evaluate our model, we conduct extensive experiments on four benchmark skeleton-based action datasets, including the large-scale challenging NTU RGB+D. The experimental results demonstrate the effectiveness of our proposed model and the improvement over the state-of-the-art.
\end{abstract}

\section{Introduction}

Human action recognition is one of the most active research topics due to its wide applications to video surveillance, robot vision, human computer interaction, etc. Since human body itself can be viewed as an articulated system of rigid bones connected by hinged joints, the actions of human body are essentially embodied in skeletal motions in the 3D space~\cite{ye2013survey}. Thereby, various skeleton based action recognition methods~\cite{liu2017global,shahroudy2016ntu,zhang2017deep} are springing up in recent years, accompanying with the progress of more accessible deep sensors.

Different from grid-shaped structures of images/videos, human skeleton, consisting of a series of joints and bones, has an irregular geometric structure. Human action may be regarded as a consecutive dynamic sequence of such irregular structures. Considering the motion complexity of the entire body skeleton, the sophisticated strategy is to separately model the trajectory of each joint or clique of joints (body part). The statistic model~\cite{xia2012view}, shape model~\cite{chaudhry2013bio}, and geometric model~\cite{vemulapalli2014human} are often used to characterize the motion trajectories. With the thriving of representation learning, the learning based methods, such as Hidden Markov Models (HMMs)~\cite{xia2012view} and recursive models (long-short term memory, LSTM)~\cite{liu2016spatio,liu2017global,song2017end}, become predominant in the dynamic representation of skeletal joints, because they dedicate more promising results on those public skeleton datasets.

Recently graph as a tool is used to represent skeletons~\cite{wang2016graph}, although graph is popular to model various structured objects. Nevertheless, the graph based method~\cite{wang2016graph} still takes the conventional technique line of graph kernel matching.
In fact, in the research fields of graph classification/matching, the graph representation (\eg, the statistic on graphlet~\cite{prvzulj2007biological}) and graph metric (\eg, graph kernel~\cite{vishwanathan2010graph}) have been historically well-studied. With recent successes of deep learning on various problems, deep representation of graphs has aroused more attention~\cite{yanardag2015deep,seo2016structured,li2017action}. But the most crucial problem is the definition/identification of homogeneous graphs because the same responses should be produced from those homogeneous graphs. To this end, Niepert \etal~\cite{niepert2016learning} used a greedy strategy of sorting those nodes within a local neighbor region and then performed convolution-like filtering on the sorted nodes. Instead of this explicit spatial definition way, Defferrard \etal~\cite{defferrard2016convolutional} introduced a deep graph method based on spectral filtering, inspired by the recent signal processing theory on graph~\cite{shuman2013emerging}. Li \etal~\cite{li2017action} introduced attention mechanism into graph convolution model. But those methods intrinsically belong to still-graph deep learning.

For dynamic graphs, some variants of recurrent neural network (RNN) are developed recently based on the traverse way of spatio-temporal graph. Jain \etal~\cite{jain2016structural} proposed a structural-RNN by casting spatio-temporal graph as a RNN mixture for the task of action prediction. Li \etal~\cite{li2015gated} proposed gated graph sequential neural network for the basic logical reasoning task. Seo~\etal~\cite{seo2016structured} fed the spatial filtered graph signals into LSTM for image generation. However, these methods do not yet absorb the essential successes of convolutional neural networks (CNN), which have changed AI landscape with breakthrough results on numerous applications. %Hence, how to apply local convolutional filtering to graph sequences is great meaningful to learn robust representation of dynamic graphs.

In this paper, we propose a spatio-temporal graph convolution (STGC) approach to represent dynamic skeletal graph sequences. To encode the graph structure data, we design multi-scale convolutional filters, each of which is composed of receptive field computation and signal mapping. The receptive field is computed from the polynomial of adjacency matrix, which makes sure the same responses for homogeneous graphs. We simultaneously perform local convolutional filtering on temporal motions and spatial structures. The temporal convolutional filtering recursively encodes motion variations while the spatial filtering extracts more robust feature of spatial structures.
In theory, the frequency responses of multi-scale convolutional filtering are equivalent to an approximate graph Flourier transform following by one linear function of feature mapping. Taking the philosophy of multi-scale convolutional filtering, we develop a recursive graph convolution model inspired by autoregressive moving average (ARMA). We theoretically analyze the stability of the proposed model and provide a theoretical upper-bound. Finally, we stack the spatio-temporal graph convolution into a deep architecture. To verify our proposed method, we conduct extensive experiments on four benchmark skeleton-based action datasets including the largest NTU RGB+D dataset~\cite{shahroudy2016ntu}. The experimental results demonstrate the effectiveness of our proposed model and a more promising direction for skeleton based action recognition.

In summary, our contributions are three folds:
\begin{itemize}
  \item Propose a spatio-temporal graph convolution approach, which assembles the successes of local convolutional filtering and sequence learning ability of recursive learning. The generic model is further extended to a deep architecture.
  \item Theoretically prove the stability of the proposed model and provide a theoretical upper-bound.
  \item Achieve the state-of-the-art performances on the four benchmark datasets including the large-scale challenging dataset NTU RGB+D.
\end{itemize}

\section{Graph Preliminary}

Consider an undirected graph $\mcG=(\mcV,\A)$ of $N$ nodes, where $\mcV=\{v_i\}_{i=1}^{N}$ is the set of nodes, $\A$ is a (weighted) adjacency matrix. The adjacency matrix $\A\in\mbR^{N\times N}$ records the connections between nodes, where if $v_i, v_j$ are not connected, then $A_{ij}=0$, otherwise $A_{ij}\neq 0$.

The graph Laplacian matrix $\L$ is defined as $\L = \D-\A$, where $\D\in\mbR^{N\times N}$ is the diagonal degree matrix with $D_{ii}=\sum_{j}A_{ij}$. A popular option is the normalized version, \ie, each weight $A_{ij}$ is multiplied by a factor $\frac{1}{\sqrt{D_{ii}D_{jj}}}$, formally, $\L^{norm} = \D^{-\frac{1}{2}}\L\D^{-\frac{1}{2}}= \I-\D^{\frac{1}{2}}\A\D^{\frac{1}{2}}$,
where $\I$ is the identity matrix. Unless otherwise specified, we use the normalized Laplacian matrix below.

As a symmetric semi-positive definite (SPD) matrix, the graph Laplacian $\L$ has a complete set of orthonormal eigenvectors $\{\bphi_1,\cdots,\bphi_N\}$ satisfying $\L\bphi_i = \lambda_i\bphi_i$, where $\{\lambda_i\}$ are nonnegative real eigenvalues. We assume all eigenvalues are ordered as $0=\lambda_1\leq\lambda_2\cdots\leq\lambda_N=\lambda_{max}$. For the normalized Laplacian matrix, we have a bound of $\lambda_{max}=2$. In matrix expression, the Laplacian matrix can be written as $\L=\bPhi\Lambda\bPhi\tp$, where $\Lambda=\diag([\lambda_1,\lambda_2,\cdots,\lambda_N])$ and $\bPhi=[\bphi_1\tp;\bphi_2\tp;\cdots;\bphi_N\tp]\tp$.

According to graph theory~\cite{shuman2013emerging}, the graph Fourier transform (GFT) of a signal $\x$ in spatial domain can be defined as $\hbx = \bPhi\tp\x$, where $\hbx$ is the produced graph frequency signal. The corresponding inverse GFT is $\x = \bPhi\hbx$. A graph filtering $\H$ is an operator that acts upon a graph signal $\x$ by amplifying or attenuating its graph Fourier coefficients as $\H\x = \sum_{n=1}^N H(\lambda_n)\whx_n \bphi_n$. The graph frequency response $H:[\lambda_{min},\lambda_{max}]\rightarrow\mbC$ controls how much $\H$ amplifies the signal spectra $H(\lambda_n)=(\bphi_n\tp\H\x)/\whx_n$.

\section{The Model}

% We focus on human action recognition from skeleton sequences with 3D joints. In this section, we first give the ......

%\subsection{Skeletal Graph Construction}

Human body skeleton is represented with a group of 3D spatial coordinates of joints. Hence, one can use a graph to depict spatial relation of skeletal joints. To keep the original coordinate information, we add attributes of nodes into the graph, \ie, $\mcG=(\mcV,\A,\X)$, where $\mcV=\{v_i\}_{i=1}^{N}$ is the set of nodes w.r.t. skeletal joints, $\A$ is a (weighted) adjacency matrix and $\X$ is a matrix of graph signals/attributes. According to the body bones between joints, we define those connected edges, and simply assign them to 1. In addition, other strategies (\eg, Gaussian kernel) may be used to produce the adjacency matrix. The signal matrix $\X=[\x_1, \x_2, \cdots, \x_d]\in\mbR^{N\times d}$ is supported on the node set $\mcV$, whose $i$-th component (or channel) $\x_i$ represents a signal of all nodes. That means, the $i$-th node $v_i\in\mcV$ is assigned with a signal vector of $d$ dimensions. For skeletal data, we define the signal of each joint with its 3D spatial coordinates, \ie, $\x_{i\cdot}:\mcV\rightarrow (x_i, y_i, z_i)$. Thus, for a dynamic graph sequence of length $T$, we may formulate it as a stream of graphs $(\mcG_1, \mcG_2, \cdots, \mcG_T)$, where $\mcG_k=(\mcV, \A_k, \X_k)$ denotes the skeleton at the $k$-th time slice.

%The signal $\x_i\in\mbR^3$ of node $v_i$ is an observed vector (a.k.a attribute), which is 3D spatial coordinates of the joint.

%For the task of skeleton-based human action recognition, a node $v_i\in\mcV$ corresponds the $i$-th joints and an edge corresponds a bone between two joints. For the
%weight between two nodes, we assign a connected edge to 1, otherwise 0. In addition, several alternative ways to weight edges are permissible, \eg, Gaussian kernel. The signal $\x_i\in\mbR^3$ of node $v_i$ is an observed vector (a.k.a attribute), which is 3D spatial coordinates of the joint.

\subsubsection{Multi-scale Graphical Convolutional Kernels}

In the standard CNN running on images, the receptive field may be conveniently defined as a local square spatial region, due to grid-shaped structure. So convolutional filtering on regular structures is accessible. On the contrary, the construction of convolutional kernels on graphs is intractable because the same filtering responses are required for homogeneous graph structures. Inspired by the graph theory~\cite{shuman2013emerging}, we resort to the adjacency matrix $\A$, which expresses the connections between nodes. As $\A^k$ exactly records the $k$-path reachable nodes, so we can construct a $k$-neighbor receptive field by defining a $k$-order polynomial of $\A$, denoted as $\psi_k(\A)$. Taking the simplest case, let $\psi_k(\A)=\A^k$, which actually describes the $k$-hop neighbor nodes. In practice, we may replace $\A$ with the Laplacian matrix $\L$ to avoid the scale effect of matrix norm during the recursive inference (see the following model). Thus, for receptive fields of $K$ scales, we define multi-scale convolutional filtering as
\begin{align}
\Z = \mcG*f = \sum_{k=0}^{K-1} \psi_k(\L)\X\V_k, \label{eqn:g_f}
\end{align}
where $\psi_k(\L)$ expresses the receptive field of the $k$-th scale and $\V_k\in\mbR^{d\times d'}$ is the corresponding signal transformation. The computation of $\psi_k(\L)\X$ weightedly summarizes the information of all nodes within the $k$-scale receptive field, which thus is homogenous-invariant for graphs.

\subsubsection{Spatio-Temporal Graph Convolution}

Inspired by the design philosophy of autoregressive moving average (ARMA)~\cite{hannan2012statistical}, we construct the spatio-temporal graphical convolutional model as follows,
\begin{align}
  \Y_{t+1} &=  \sum_{k=0}^{K_1-1} \psi_k(\L)\Y_{t}\W_k + \X_{t}\V_0, \label{eqn:mdl:yta2}\\
  \O_{t+1} &= \Y_{t+1} + \sum_{k=1}^{K_2-1}\psi_k(\L)\X_t\V_k, \label{eqn:mdl:zta2}
\end{align}
where $\psi_k(\cdot)$ is a receptive field function on the $k$-th scale, $\{\W_k\in\mbR^{d'\times d'}, \V_k\in\mbR^{d\times d'}\}$ are the signal transformation matrices with regard to the $k$-scale,
$K_1$ and $K_2$ are the number of kernels respectively in the temporal and spatial domain. In the above model, $\Y=[\y_1,\cdots,\y_{d'}], \O=[\o_1,\cdots,\o_{d'}]$ can be respectively viewed as the hidden state and the output state. Along time slices, signals are recursively regressed with local convolutional kernels in Eqn.~(\ref{eqn:mdl:yta2}), thus motion variations can be sequentially encoded. The output signals~in Eqn.~(\ref{eqn:mdl:zta2}) combine spatially convolutional graph signals as well as dynamic temporal signals. Moreover, each output signal $\o_i$ is dependent on all input signals $\{\x_1,\cdots,\x_d\}$. Specifically, when signals are independent on each other, and the spatio-temporal convolutional filters are separately operated on a channel of signals, then the dynamic graph convolution model can be written as
\begin{align}
  \Y_{t+1} &=  \sum_{k=0}^{K_1-1} \psi_k(\L)\Y_{t}\diag(\w_k) + \X_{t}\diag(\v_0), \label{eqn:mdl:yta1}\\
  \O_{t+1} &= \Y_{t+1} + \sum_{k=1}^{K_2-1}\psi_k(\L)\X_t\diag(\v_k), \label{eqn:mdl:zta1}
\end{align}
where $\w_k = [w_{k1}, \cdots, w_{kd}]\tp, \v_k = [v_{k1}, \cdots, v_{kd}]\tp$ are mapping parameters, and $w_{ki}, v_{ki}$ are associated to the $k$-th scale of the $i$-th signal. Note here the dimension of output is assumed to be the input dimension. It is easy to extend into $d'\neq d$ in the case of signal independency.

\subsubsection{Frequency Domain Analysis}

As the receptive field function $\psi_k(\cdot)$ is a polynomial expression, we can derive $\psi_k(\L)=\psi_k(\bphi\diag(\blam)\bphi\tp)=\bphi\psi_k(\diag(\blam))\bphi\tp$. Thus the frequency response of convolution filtering in Eqn.~(\ref{eqn:g_f}) can be written as
\begin{align}
\hbZ = \sum_{k=0}^{K-1}\diag(\psi_k(\blam))\hbX\V_k . \label{eqn:htz0}
\end{align}
In particular, if signals are independent as expressed in Eqn.~(\ref{eqn:mdl:yta1}) and (\ref{eqn:mdl:zta1}), then we can obtain the frequency responses of graphical signals as $
\hbZ = \sum_{k=0}^{K-1}(\psi_k(\blam)\v_k\tp)\odot\hbX$, each frequency response of which is formally  $\whZ_{ij}=\sum_{k=0}^{K-1}(v_{kj}\psi_k(\lambda_i))\whX_{ij}$. Therefore, when signals are independent, the multi-scale convolutional filtering on graphs may be regarded as a ($K$-1)-order polynomial approximation of graph Flourier transform, if let $H(\lambda_i)=\sum_{k=0}^{K-1}v_{ki}\psi_k(\lambda_i)$.

In the signal dependency case, we can decompose $\V_k$ into a diagonal matrix multiplied by a general matrix, \ie, $\V_k = \diag(\balpha_k)\tbV_k$. Then the frequency response of convolution filtering in Eqn.~(\ref{eqn:htz0}) can be written as
$\hbZ = \sum_{k=0}^{K-1}((\psi_k(\blam)\balpha_k\tp)\odot\hbX)\tbV_k$.
Therefore, in the general case of signal dependency, the calculation of frequency responses may be understood as two steps: (i) perform a polynomial approximation of GFT on each input signal, and (ii) transform multi-channel signals by the new mappings $\{\tbV_k\}$.

\subsubsection{Stability Analysis}

For any sequence of graph realizations $\{\mcG_1, \mcG_2, \cdots\}$, we can prove the stability of the recursive model (Eqn.~(\ref{eqn:mdl:yta2})$\sim$(\ref{eqn:mdl:zta2})), which is summarized in the following theory.

\begin{thm}\label{thm:1}
Suppose the Laplacian matrix $\L$ has the eigenvalue decomposition $\L=\bPhi\bLam\bPhi\tp$, and $\psi_k(\cdot)$ is a $k$-order polynomial function satifying $\|\psi_k(\L)\|_2 \leq 1$. For the signal mappings $\{\W_0,\cdots, \W_{K_1-1}\}$, if their diagonal elements are non-negative, i.e., $W_{k,ii}\geq 0$, and $\sum_{k=0}^{K_1-1}\|\W_{k}\|_\infty<1$, the frequency response $vec(\hbO)$ in Eqn.~(\ref{eqn:mdl:zta2}) will converge to
\begin{align}
&\lim_{t\rightarrow\infty}vec(\hbO_t)=\mcT\cdot vec(\hbX_{t-1}), \label{eqn:lim_vec_hbv_t}\\
\mcT &= (\I\!-\!\Gamma_0^{K_1}(\W,\bLam))^{-1}\Gamma_0^{1}(\V,\bLam) \!+\! \Gamma_1^{K_2}(\V, \bLam),\label{eqn:lim_T}
\end{align}
where $\Gamma_a^{b}(\W,\bLam)=\sum_{k=a}^{b-1} (\W_k\tp\otimes\psi_k(\bLam))$ (similar for $\Gamma_0^{K_2}$).
Moreover, the transformation function $\mcT$ has an upper-bound:
\begin{align}
\|\mcT\|_2 < \frac{\|\V_0\|_\infty}{1-\sum_{k=0}^{K_1-1}\|\W_k\|_\infty} + \sum_{k=1}^{K_2}\|\V_k\|_\infty.
\end{align}
\end{thm}

\begin{proof}
As $\psi_k(\cdot)$ is a polynomial function, we can transform the recursive model (Eqn.~(\ref{eqn:mdl:yta2})$\sim$(\ref{eqn:mdl:zta2})) from spatial domain into frequency domain by using $\hbY=\bPhi\tp\Y$ and $\hbX=\bPhi\tp\X$,
\begin{align}
\hbY_{t+1} &=  \sum_{k=0}^{K_1-1} \psi_k(\bLam)\hbY_{t}\W_k + \hbX_{t}\V_0, \label{eqn:hty_ta1}\\
\hbO_{t+1} &= \hbY_{t+1} + \sum_{k=1}^{K_2-1}\psi_k(\bLam)\hbX_t\V_k \label{eqn_htz_ta1}.
\end{align}
By using the abbreviated notation $\Gamma_a^b(\cdot,\cdot)$ and $\psi_0(\cdot)=\I$, we can derive a vertorized style of Eqn.~(\ref{eqn:hty_ta1}) as
\begin{align}
  vec(\hbY_{t+1})  %vec(\sum_{k=0}^{K_1-1} \psi_k(\Lambda)\hbY_{t}\W_k) + vec(\hbX_{t}\V_0) \nonumber\\
             %& = \sum_{k=0}^{K_1-1} (\W_k\tp\otimes\psi_k(\Lambda))vec(\hbY_{t}) + (\V_0\otimes\I)vec(\hbX_{t}) \nonumber\\
              &= \Gamma_0^{K_1}(\W,\bLam)vec(\hbY_{t}) + \Gamma_0^{1}(\V,\bLam)vec(\hbX_{t})\nonumber\\
  &            = \sum_{\tau=0}^{t}((\Gamma_0^{K_1}(\W,\bLam))^\tau)\Gamma_0^{1}(\V,\bLam)vec(\hbX_{t}) \nonumber\\
             & \quad+(\Gamma_0^{K_1}(\W,\bLam))^{t+1} vec(\hbY_{0}). \label{eqn:vec(hby_ta1)}
\end{align}
Next we need to prove $\|\Gamma_0^{K_1}(\W,\bLam)\|_2<1$. The matrix $\Gamma_0^{K_1}(\W,\bLam)$ is a block matrix of $d'\times d'$ blocks, each of which is a diagonal matrix. Let $\A^{(ij)} = \sum_{k=0}^{K_1-1}W_{k,ij}\psi_k(\bLam)$ denote the $(i,j)$-th block diagonal matrix, where $W_{k,ij}$ denotes the $(i,j)$-th element of the matrix $\W_k$. As $\psi_k(\bLam)$ is a diagonal matrix, we denote  $\psi_k(\bLam)=\diag([\tlam_{k1},\tlam_{k2},\cdots,\tlam_{kn}])$ for simplification.
Since $\|\psi_k(\L)\|_2 \leq 1$, we have $\tlam_{ki}\in[-1,1]$. And, since $\sum_{k=0}^{K_1-1}\|\W_{k}\|_\infty<1$, we have
\begin{align}
&\|\Gamma_0^{K_1}(\W,\bLam)\|_2 \!\!\leq\!\! \max_{i,r} \sum_j|\A^{(ij)}_{rr}|
 \!\!=\!\! \max_{i,r} \sum_j|\!\!\sum_{k=0}^{K_1-1}W_{k,ij}\tlam_{kr}| \nonumber\\
 &\qquad\qquad\qquad \leq \max_{i,r} \sum_j \sum_{k=0}^{K_1-1}|W_{k,ij}| < 1.
\end{align}
Therefore, when $t\rightarrow\infty$, we can derive $vec(\hbY_{t+1})$ as
\begin{align}
\!\!vec(\hbY_{t+1})
       \!=\! (\I-\Gamma_0^{K_1}(\W,\bLam))^{-1}\Gamma_0^{1}(\V,\bLam)vec(\hbX_{t}).
\end{align}
After a simple algebra calculation on Eqn.~(\ref{eqn_htz_ta1}), we can reach the conclusion of Eqn.~(\ref{eqn:lim_vec_hbv_t})$\sim$(\ref{eqn:lim_T}).

Now we prove the bound of the transformation function $\mcT$. First, we can derive the following two inequations,
\begin{align}
&\|(\I-\A_{ii})^{-1}\|_\infty^{-1} = \|(\I-\sum_{k=0}^{K_1-1} W_{k,ii}\psi_k(\bLam))^{-1}\|_\infty^{-1} \nonumber\\
        &=\min_l |1-\sum_{k=0}^{K_1-1}W_{k,ii}\tlam_{kl}| > =1-|\sum_{k=0}^{K_1-1}W_{k,ii}|,\\
&\sum_{j\neq i}\|\A_{ij}\|_\infty=\sum_{j\neq i}\|\sum_{k=0}^{K_1-1} W_{k,ij}\psi_k(\bLam)\|_\infty \nonumber\\
        \!\!&= \sum_{j\neq i} \max_{l}|\sum_{k=0}^{K_1-1} W_{k,ij}\tlam_{kl}|
        <\! \sum_{j\neq i}\! \sum_{k=0}^{K_1-1} |W_{k,ij}|.
\end{align}
Note the above derivation uses the conditions $W_{k,ii}\geq 0$ and $\tlam_{kl}\in[-1,1]$. When $\sum_{k=0}^{K-1}\|\W_{k}\|_\infty<1$, we can further have $\|(\I-\A_{ii})^{-1}\|_\infty^{-1}>\sum_{j\neq i}\|\A_{ij}\|_\infty$ for $\forall i=1,\cdots, d'$. So $\I-\Gamma_0^{K_1}(\W,\bLam)$ is strictly block diagonal dominant (SBDD). According to Ahlberg-Nilson-Varah bound of SSDD matrix~\cite{moravca2007upper}, after a series of algebra derivations, we can reach the bound:
\begin{align}
\|(\I-\Gamma(\W,\bLam))^{-1}\|_\infty  < \frac{1}{1-\sum_{k=0}^{K_1}\|\W_k\|_\infty}.
\end{align}
Similarly, $\|\Gamma_0^1(\V,\bLam)\|_\infty \leq \|\V_0\|_\infty,~
\|\Gamma_1^{K_2}(\V, \bLam)\|_\infty \leq \sum_{k=1}^{K_2}\|\V_k\|_\infty.
$
Finally we obtain an upper bound of $\mcT$.
\end{proof}

When signals are independent, \ie, the recursive model takes Eqn.~(\ref{eqn:mdl:yta1})$\sim$(\ref{eqn:mdl:zta1}), we can have the following corollary based on the above theory.
\begin{prop}
Suppose the Laplacian matrix $\L$ has the eigenvalue decomposition $\L=\bPhi\diag(\blam)\bPhi\tp=\bPhi\diag([\lambda_1,\cdots,\lambda_n])\bPhi\tp$, and $\psi_k(\L)$ is a $k$-order polynomial. If signals are independent, i.e., taking the recursive model of Eqn.~(\ref{eqn:mdl:yta1})$\sim$(\ref{eqn:mdl:zta1}), for $\forall i=1,\cdots,d, j=1,\cdots, n$, and $w_{ki}\geq 0$ and $|\sum_{k=0}^{K_1-1}w_{ki}\psi_k(\lambda_j)|<1$, then the frequency response $\hbO$ in Eqn.~(\ref{eqn:mdl:zta1}) will converge to
\begin{align}
&\lim_{t\rightarrow\infty}\hbO_{t+1}=\mcT\odot \hbX_t, \label{eqn:lim_hbz_t}\\
\mcT &= \frac{\1\v_0\tp}{1-\sum_{k=0}^{K_1-1} \psi_k(\blam)\w_k\tp}+ \sum_{k=1}^{K_2-1} \psi_k(\blam)\v_k\tp,
\end{align}
where $-,\odot$ respectively denote the elementwise division and multiplication.
Further, if $|\psi_k(\lambda_j)|<1$, and considering $\mcT$ is the elementwise multiplication, then for $\forall (i,j)$ term of $\mcT$, we have an upper-bound:
\begin{align}
\|\mcT_{ij}\|_2 < \frac{\|\v_0\|_\infty}{1-\|\sum_{k=0}^{K_1-1}\w_k\|_\infty} + \sum_{k=1}^{K_2}\|\v_k\|_\infty.\label{eqn:bound2}
\end{align}
\end{prop}
\begin{proof}
According to Eqn.~(\ref{eqn:vec(hby_ta1)}) in Theory \ref{thm:1}, let $\Upsilon_a^b(\w,\blam)=\sum_{k=a}^{b-1}\psi_k(\blam)\w_k\tp$, we can have
\begin{align}
  \hbY_{t+1}
              = \sum_{\tau=0}^{t}((\Upsilon_0^{K_1}(\w,\blam))^\tau)\odot\Upsilon_0^{1}(\v,\blam)\odot\hbX_{t}  \nonumber\\
              \qquad\qquad+(\Upsilon_0^{K_1}(\w,\blam))^{t+1} \odot\hbY_{0}. \label{eqn:vec(hby_ta1)2}
\end{align}
Since $|\sum_{k=0}^{K_1-1}w_{ki}\psi_k(\lambda_j)|<1$, when $t\rightarrow\infty$, we can have
\begin{align}
   \hbY_{t+1}
       \!=\! \frac{\Upsilon_0^{1}(\v,\blam)}{1-\Upsilon_0^{K_1}(\w,\blam)}\odot\hbX_{t}.
\end{align}
Further, Eqn.~(\ref{eqn:lim_hbz_t}) can be obtained after simple derivations. The upper-bound of $\mcT$ can also be easily obtained..
\end{proof}

In practice we can take some normalization strategies on $\{\L, \W\}$ to make them satisfy these preconditions, in order to guarantee model stability.

\subsubsection{Deep Stacking}

The above recursive convolutional model can be easily extended into a deep architecture. Taking the recursive model as one basic layer, we may stack it into a multi-layer network architecture, in which the output signal $\O$ at the bottom layer is used as the input of the top layer. With the increase of layers, the receptive field size of convolutional kernels can become larger, thus the topper layer can abstract more global information. As observed from our experiments, this deep architecture (named deep STGC) can improve the performance of skeleton based action recognition. Besides, we may insert the recursive model into other networks as one basic unit to form a mixture network, according to the requirement of solved problems.

\section{Experiments}

\begin{table*}[t]
  \caption{Performance comparisons with different configurations for our model.}
  \label{tab:results}
  \centering
  \scalebox{0.8}{
    \begin{tabular}{l|c|c|c|c|c|c|c|c|c}
      \hline\hline
      \multirow{3}{*}{Method}&\multirow{2}{*}{Florence}&\multicolumn{2}{c|}{HDM05}&\multicolumn{4}{c|}{LSC}&\multicolumn{2}{c}{NTU}\\
      \cline{3-10}&&Protocol 1&Protocol 2&\multicolumn{2}{c|}{Cross Sample} &\multicolumn{2}{c|}{Cross Subject}&Cross View&Cross Subject\\
      \cline{2-10}&Accuracy&Accuracy&Accuracy&Precision&Recall&Precision&Recall&Accuracy&Accuracy\\
      \cline{1-10}
      STGC (w/o $\L$)           & 98.14\% & 73.82\% & 81.92\%$\pm$1.04 & 82.51\% & 80.80\% & 82.44\% & 81.12\% & 80.87\% & 70.29\% \\
      STGC (w/ $\L$)            & 98.60\% & 75.76\% & 82.24\%$\pm$1.18 & 83.65\% & 82.04\% & 83.63\% & 80.24\% & 81.11\% & 71.00\% \\
      STGC$_K$ (indep.)         & 98.14\% & 77.08\% & 82.49\%$\pm$1.49 & 84.70\% & 82.60\% & 82.08\% & 78.17\% & 79.71\% & 70.62\% \\
      STGC$_K$ (dep.)           & 97.67\% & 77.99\% & 83.40\%$\pm$1.30 & 86.44\% & 85.10\% & 84.00\% & 82.12\% & 81.84\% & 72.09\% \\
      Deep STGC$_K$ (indep.)    & 98.60\% & 78.19\% & \textbf{86.17\%$\pm$1.25} & 86.73\% & 86.18\% & 84.58\% & 81.11\% & 83.57\% & 73.60\% \\
      Deep STGC$_K$ (dep.)      & \textbf{99.07\%} & \textbf{78.68\%} & 85.29\%$\pm$1.33 & \textbf{88.11\%} & \textbf{87.03\%} & \textbf{85.44\%} & \textbf{83.42\%} & \textbf{86.28\%} & \textbf{74.85\%} \\
      \hline\hline
    \end{tabular}}
\end{table*}

We conduct experiments on four public skeleton based action datasets: Florence 3D~\cite{seidenari2013recognizing}, HDM05~\cite{muller2007documentation}, Large Scale Combined dataset~\cite{zhang2016} and NTU RGB+D~\cite{shahroudy2016ntu}. To investigate the effectiveness of our model, we conduct extensive experiments with different configurations listed as follows:
\begin{itemize}
  \item {STGC (w/o $\L$)} is the standard baseline without considering neighbor nodes, \ie, only performing channel mapping;
  \item {STGC w/ $\L$} uses convolutional kernels with the receptive filed of one hop neighborhood;
  \item {STGC$_K$ (indep.)} is the multi-scale filtering version under the assumption of signal independency.
  \item {STGC$_K$ (dep.)} is the general multi-scale filtering version, where signals are dependent on each other.
  \item {Deep STGC$_K$ (indep.)} is the deep scheme by stacking {STGC$_K$ (indep.)}.
  \item {Deep STGC$_K$ (dep.)} is the deep scheme by stacking {STGC$_K$ (dep.)}
\end{itemize}

\subsection{Datasets and Settings}

\subsubsection{Florence 3D (Florence)} This dataset was collected from a stationary Kinect, and each body skeleton was recorded with only 15 joints. It contains 215 action sequences of 10 subjects with 9 actions: {wave, drink from a bottle, answer phone, clap, tight lace, sit down, stand up, read watch, and bow}. Due to a few skeletal joints, some types of actions are difficult to distinguish, such as drink from a bottle, answer phone and read watch.
We follow the standard experimental settings to perform leave-one-subject-out cross-validation~\cite{wang2016graph}.

\subsubsection{HDM05}

This dataset was captured by using an optical marker-based Vicon system, and contains 2337 action sequences of 130 motion classes, which are acted by 5 non-professional actors named ``bd", ``bk", ``dg", ``mm" and ``tr". Each skeleton data is represented with 31 joints. Until now, this dataset should involve the most skeleton-based action categories to the best of our knowledge. Due to the intra-class variations and large number of motion classes, this dataset is challenging in action recognition.
To compare with those previous literatures, we conduct two types of experiments by following two widely-used protocols. Firstly, we use two subjects ``bd" and ``mm" for training and the remaining three for testing~\cite{wang2015beyond}. Secondly, to fairly compare the current deep learning methods, we conduct 10 random evaluations, each of which randomly selects half of the sequences for training and the rest for testing~\cite{huang2017riemannian}.

\subsubsection{Large Scale Combined (LSC)}

This dataset combines nine publicly available datasets
\cite{li2010action,wang2015convnets,wang2016action,xia2012view,wang2012mining,oreifej2013hon4d,koppula2013learning,sung2012unstructured,bloom2013dynamic,bloom2012g3d,ni2011rgbd,chen2015utd}, and form a complex action dataset with 88 actions. As each individual dataset has its own characteristics in action execution manners, backgrounds, acting positions, view angles, resolutions, and sensor types, the combination of a large number of action classes makes the dataset more challenging in suffering large intra-class variation compared to each individual dataset.
Following Zhang~\etal~\cite{zhang2016}, we conduct experiments using two standard settings, \ie, random cross subject evaluation and random cross sample evaluation. For each action, half of the subjects/samples are randomly selected for training while the rest for testing.

\subsubsection{NTU RGB+D (NTU)}

This dataset is collected by Kinect v2 cameras from different views. It consists of 56880 sequences for 60 distinct actions, including various of daily actions and pair actions performed by 40 subjects. The skeleton data is represented by 25 joints. As far as we know, this dataset is currently the largest skeleton-based action recognition dataset. The large intra-class and view point variations make this dataset great challenging. Meanwhile, a large amount of samples will bring a new challenge to the current skeleton-based action recognition methods.
We follow the two types of standard evaluation protocols~\cite{shahroudy2016ntu}, \ie, cross-view evaluation, cross-subject evaluation, to perform experiments.

\begin{figure}[t]
  \centering
  \includegraphics[width=2.8in]{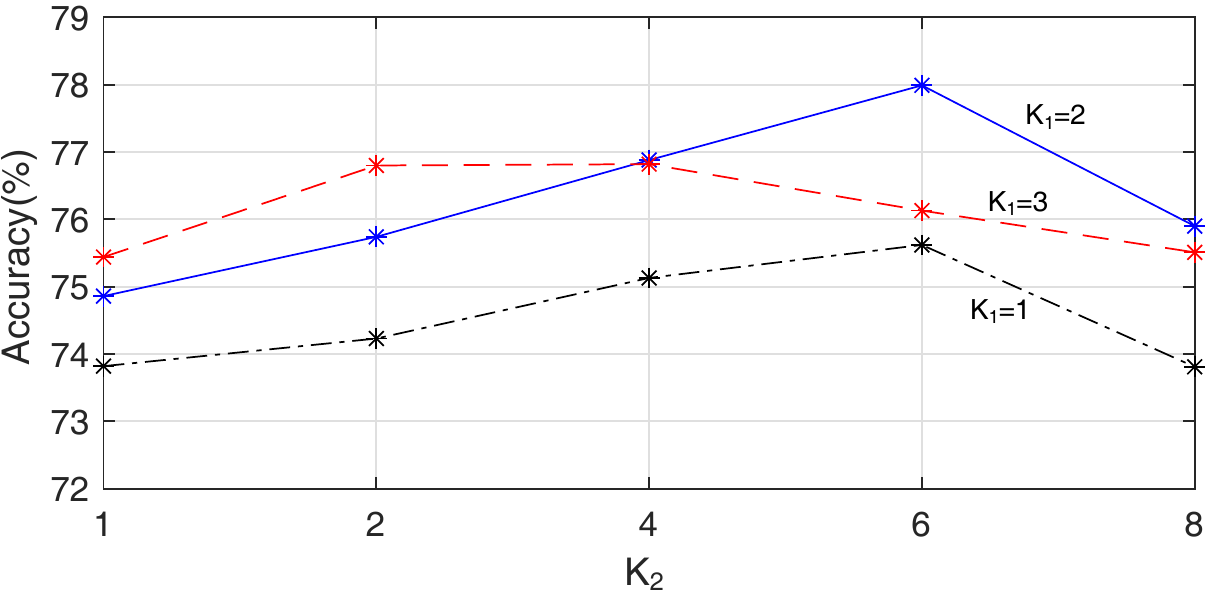}\\
  \caption{Comparisons of different convolutional kernel scale $K_1,K_2$ on HDM05 (Protocol 1) for our SGTC.}
  \label{fig:params}
\end{figure}

\subsection{Implementation Details}

\subsubsection{Data Processing}

As skeletal data is usually captured from multi-view points and human actions are independent on the user coordinate system, we modify the origin of the coordinate system as the orthocenter of joints for each frame of skeleton, \ie,
$\mcO=\frac{1}{N}\sum_{i=1}^{N} \x_i$, where $\x_i\in \mbR^3$ is a 3D coordinate of the $i$-th joint, $N$ is the number of joints. Specially, for NTU, we preprocess the joint coordinates in a way similar to Shahroudy~\etal~\cite{shahroudy2016ntu}.
To enhance the robustness of model training, we perform data augmentation as widely used in previous deep learning literature~\cite{liu2016spatio,shahroudy2016ntu}. Concretely, for each action sequence, we split the sequence into several equal sized subsequences, here 12 segments, and then pick one frame from each segment randomly to generate a large amount of training sequences. In addition, we randomly scale the skeletons by multiplying a factor in [0.98, 1.02] for the sake of the adaptive capability of scaling.

\subsubsection{Model Configuration}

For the undirected attribute graph, we simply design edge connections according to human bones. If two joints are bridged with a bone, the edge weight is assigned to 1, otherwise 0. The signals of each node are set to its 3D coordinate. For the receptive field function, we use the simplest polynomial term, \ie, $\psi_k(\L)=\L^k$, which represents $k$-hop neighbors. According to Theory \ref{thm:1}, we normalize/clip $\L=\L/\lambda_{max}$ ($\lambda_{max}=2$), $\W_{k,ii}\geq 0, \sum_k|\W_k|\1<\1$ after the gradient update at each iteration, to make sure the model stability. The outputs of recursive model are concatenated into a softmax layer for classification. In the default case, the dimension of output signals is $d'=32$. For deep STGC, we empirically observe that stacking two layers is good enough to these datasets, thus we only employ two-layer network for Deep STGC. The outputs of each layer in deep STGC are $32, 64$ dimensions. The most important factor of our model is the scale of convolutional kernels, which is analyzed in the next Section~\ref{sec:receptive}.

\begin{table}[t]
  \caption{Comparisons on Florence 3D dataset.}
  \label{tab:Florence}
  \centering
  \scalebox{0.7}{
    \begin{tabular}{lc}
      \hline\hline
      Method  &  Accuracy \\ \hline
      Multi-part Bag-of-Poses~\cite{seidenari2013recognizing} & 82.00\% \\
      Riemannian Manifold~\cite{devanne20153}                 & 87.04\% \\
      Lie Group~\cite{vemulapalli2014human}                   & 90.88\% \\
      Graph-Based~\cite{wang2016graph}                        & 91.63\% \\
      MIMTL~\cite{yang2017discriminative}                     & 95.29\% \\
      P-LSTM~\cite{shahroudy2016ntu}                          & 95.35\% \\ \hline
      STGC$_K$                                                & 97.67\% \\
      Deep STGC$_K$                                           & {\bf 99.07\%} \\
      \hline\hline
    \end{tabular}}
\end{table}

\begin{table}[t]
  \caption{Comparisons on HDM05 dataset.}
  \label{tab:HDM05}
  \centering
  \setlength{\tabcolsep}{4pt}
  \scalebox{0.7}{
    \begin{tabular}{lcc}
      \hline\hline
      Method & Accuracy & Accuracy \\ \hline
      RSR-ML~\cite{harandi2014manifold}       & 40.0\% & - \\
      Cov-RP~\cite{tuzel2006region}           & 58.9\% & - \\
      Ker-RP~\cite{wang2015beyond}            & 66.2\% & - \\
      SPDNet~\cite{huang2017riemannian}       & - & 61.45\%$\pm$1.12\\
      Lie Group~\cite{vemulapalli2014human}   & - & 70.26\%$\pm$2.89\\
      LieNet~\cite{huang2016deep}             & - & 75.78\%$\pm$2.26\\
      P-LSTM~\cite{shahroudy2016ntu}          & 70.4\% & 73.42\%$\pm$2.05 \\ \hline
      STGC$_K$                                & 78.0\% & 83.40\%$\pm$1.30 \\
      Deep STGC$_K$                           & {\bf 78.7\%} & {\bf85.29\%$\pm$1.33} \\
      \hline\hline
    \end{tabular}}
\end{table}

\begin{table}[t]
  \caption{Comparisons on Large Scale Combined dataset.}
  \centering
  \label{tab:LSC}
  \scalebox{0.7}{
    \begin{tabular}{lcccc}
      \hline\hline
      \multirow{2}{*}{Method} &
      \multicolumn{2}{c}{Cross Sample} & \multicolumn{2}{c}{Cross Subject} \cr\cline{2-5}
      & Precision & Recall & Precision & Recall \cr
      \hline
      HON4D~\cite{oreifej2013hon4d}         & 84.6\% & 84.1\% & 63.1\% & 59.3\% \\
      Dynamic Skeletons~\cite{hu2015jointly}& 85.9\% & 85.6\% & 74.5\% & 73.7\% \\
      P-LSTM~\cite{shahroudy2016ntu}        & 84.2\% & 84.9\% & 76.3\% & 74.6\% \\ \hline
      STGC$_K$                              & 86.4\% & 85.1\% & 84.0\% & 82.1\% \\
      Deep STGC$_K$                         &{\bf 88.1\%}&{\bf 87.0\%}&{\bf 85.4\%}&{\bf 83.4\%}\\
      \hline\hline
    \end{tabular}}
\end{table}

\begin{table}[t]
  \caption{Comparisons on NTU RGB+D dataset.}
  \label{tab:NTU}
  \centering
  \scalebox{0.7}{
    \begin{tabular}{lcc}
      \hline\hline
      Method & \tabincell{c}{Cross\\View} & \tabincell{c}{Cross\\Subject} \\
      \hline
      Lie Group~\cite{vemulapalli2014human}               & 52.76\% & 50.08\% \\
      Dynamic Skeletons~\cite{hu2015jointly}              & 65.22\% & 60.23\% \\
      HBRNN~\cite{du2015hierarchical}                     & 63.97\% & 59.07\% \\
      LieNet~\cite{huang2016deep}                         & 66.95\% & 61.37\% \\
      Deep LSTM~\cite{shahroudy2016ntu}                   & 67.29\% & 60.69\% \\
      P-LSTM~\cite{shahroudy2016ntu}                      & 70.27\% & 62.93\% \\
      ST-LSTM~\cite{liu2016spatio}                        & 77.70\% & 69.20\% \\
      STA-LSTM~\cite{song2017end}                         & 81.20\% & 73.40\% \\
      GCA-LSTM~\cite{liu2017global}                       & 82.80\% & 74.40\% \\
      Geometric Features~\cite{zhang2017geometric}        & 82.39\% & 70.26\% \\
      Clips + CNN + MTLN~\cite{ke2017new}*                & 84.83\% & \textbf{79.57\%} \\
      View invariant(Raw Samples)~\cite{liu2017enhanced}* & 82.56\% & 75.97\% \\ \hline
      STGC$_K$                                            & 81.84\% & 72.09\% \\
      Deep STGC$_K$                                       & \textbf{86.28\%} & 74.85\% \\
      \hline\hline
    \end{tabular}}
  \flushleft{\quad\small\textit{* They used the VGG-19/CNN model pre-trained on ImageNet \\ ~~~~~~ after converting skeletons into images.}}
\end{table}

\subsection{Selection of Convolutional Kernel Size}\label{sec:receptive}

In the STGC$_K$ model, the parameters $\{K_1,K_2\}$ control the receptive field sizes in temporal domain and spatial domain. When increasing $K_1, K_2$, the local filtering region will cover the farther hoping neighbors. To check the effect of different scales, we conduct an experiment on HDM05 dataset by searching the temporal kernel scale $K_1$ in $\{1,2,3\}$ and the spatial kernel scale $K_2$ in $\{1, 2,4,6,8\}$. Considering the continuous convolution filtering is performed along time slices like a stacking CNN, in practice we should employ a smaller $K_1$ than $K_2$. The cross-comparison results are reported in Figure~\ref{fig:params}. The best performance is obtained at $K_1=2$ and $K_2=6$, which are used as the default parameters. Note that, $K=1$ means STGC(w/o $\L$), which convolves only on the node itself without any neighbors. If $K_1=1$ and $K_2 > 1$, only spatial filtering is taken.  Conversely, $K_1>1$ and $K_2 = 1$ for only temporal filtering. We can observe that the spatial and temporal filtering together contribute the gain for action recognition.

\subsection{Verification of STGC Structure}

For our proposed model STGC itself, we conduct a series of experiments with different configurations on four benchmark datasets. The results are summarized in Table~\ref{tab:results}. They include six configurations: STGC(w/o $\L$), STGC(w/ $\L$), STGC$_K$(indep.), STGC$_K$(dep.), deep STGC$_K$(indep.) and deep STGC$_K$(dep.). The standard baseline should be STGC(w/o $\L$), which doesn't use any adjacent relationship, \ie, implementing a multi-channel mapping like $1\times 1$ convolution on images. When introducing the convolutional kernel with 1-neighborhood, \ie, STGC(w/ $\L$), the performance is improved. As discussed in Section~\ref{sec:receptive}, we set $K_1=2$ and $K_2=6$ as default for multi-scale STGC$_K$(indep.), STGC$_K$(dep.), where the former is the version of independent signals while the latter is the general one. STGC$_K$(dep.) is slightly superior to STGC$_K$(indep.) due to the consideration of signal interaction. When extending STGC into the deep architecture of two layers, we can achieve the best performance on all four benchmark datasets.

\subsection{Comparisons with State-of-the-Art}

We compare state-of-the-art methods on Florence, HDM05, LSC, NTU, respectively, which are summarized in Table~\ref{tab:Florence},~\ref{tab:HDM05},~\ref{tab:LSC} and~\ref{tab:NTU}. As observed from these results, we have the following observations.

\textit{The proposed spatio-temporal graph convolution method is superior to the recent graph-based method~\cite{wang2016graph} and LSTM-based methods~\cite{shahroudy2016ntu,liu2016spatio,liu2017global,song2017end}}. As shown in Table~\ref{tab:Florence}, our STGC has a large improvement (more than 7\%) in contrast to the graph-based work~\cite{wang2016graph}. In principle, our STGC is very different from this work~\cite{wang2016graph}, although graph is used for both. Our method falls into a recursively convolutional architecture, while the work~\cite{wang2016graph} follows the conventional graph kernel matching technique. Also, different from those LSTM-based methods, which only model dynamics of sequences by revising LSTM, our method absorbs the success of convolutional filtering into a recursive learning with a theoretical guarantee.

\textit{Our proposed STGC improves the current state-of-the-art on most datasets}. On the Florence dataset, our method achieves a nearly perfect performance 99.07\%. On the current largest dataset NTU, under the same recursive idea, the performance is pushed to the higher 86.28\% and 74.85\%, from 82.80\% and 74.40\% for GCA-LSTM. Recently, the CNN-based methods~\cite{ke2017new,liu2017enhanced} converted skeletons into images and then employed the sophisticated CNN feature extraction techniques by using the pre-training on ImageNet~\cite{deng2009imagenet}. Even so, without the use of extra training data, our STGC is still more competitive over them.

\textit{Deep learning based methods are more effective than those shallow learning methods}. The advanced nonlinear dynamic networks, variations of LSTM~\cite{shahroudy2016ntu,liu2016spatio,liu2017global,song2017end} and CNN-based models~\cite{ke2017new,liu2017enhanced}, largely improve the action recognition performance, due to their robust representation ability. For the conventional matrix-based descriptors (\eg, covariance or its variants), although the deep manifold learning strategies~\cite{huang2016deep,huang2017riemannian} are developed recently, the matrix-based representations limit their representation capability because the only second-order statistic relationship of skeletal joints is preserved, whereas first-order statistics is also informative~\cite{ranzato2010modeling}.

\textit{Different datasets have different performances}. Florence 3D is the simplest dataset with 215 sequences and 9 action classes, thus most methods obtain a higher accuracy. The most difficult dataset should be the largest dataset NTU, which consists of 56880 sequences and covers various of daily actions and pair actions. The cross subject accuracy is still less than 80\% due to various entangled actions.
\textit{Cross subject is more difficult than cross view or cross sample}. The phenomenon is observed from Table~\ref{tab:LSC}, Table~\ref{tab:NTU}, and Table~\ref{tab:HDM05} (the left/right column w.r.t cross subject/cross sample). It is easy to understand, in the cross subject task, more unforeseeable information exists in the testing set, compared to the other tasks.

\section{Conclusion}

We proposed a spatio-temporal graph convolution approach for combining the successes of local convolution filtering and the recursive learning power of autoregressive moving average. To locally filter on spatio-temporal structures, we introduced multi-scale graphical convolutional kernels, which were composed of the receptive field matrices defined by polynomials of adjacency matrix, and signal mappings. The multi-scale convolutional kernels were simultaneously performed on hidden states of sequences for encoding the motion variations and input state for extracting spatial graphical feature. In theory, we proved the convergency of the proposed model and provided an upper-bound. Moreover, we extended the basic model into a multi-layer deep architecture. We verified the representation ability of the proposed STGC and its deep version. We also demonstrated the improvements with STGC on four public skeletal datasets, including the current largest NTU RGB+D dataset. As a generic model, the proposed model may be generalized into many problems modeled by dynamic graphs, which will be one of our  future work.

\section*{Acknowledgement}

This work was supported in part by the National Basic Research Program of China under Grant 2015CB351704 and 2014CB349303, in part by the National Natural Science Foundation of China under Grant Nos. 61772276, 61602244, 61472187, 61231002 and 91420201, and program for Changjiang Scholars.

\begin{small}
  \bibliographystyle{aaai}
  \bibliography{reference}
\end{small}

\end{document}